\documentclass[conference]{IEEEtran}
\IEEEoverridecommandlockouts
\usepackage{cite}
\usepackage{amsmath,amssymb,amsfonts}
\usepackage{algorithmic}
\usepackage{graphicx}
\usepackage{textcomp}
\usepackage{xcolor}
\usepackage{bm,mathtools,amsmath,scrextend, array, setspace, algorithm, textcomp, amssymb, comment, times, amsthm, subfig}
\newcolumntype{M}[1]{>{\centering\arraybackslash}m{#1}}
\def\BibTeX{{\rm B\kern-.05em{\sc i\kern-.025em b}\kern-.08em
    T\kern-.1667em\lower.7ex\hbox{E}\kern-.125emX}}
\newtheorem{prop}{Proposition}
\begin{document}

\title{Fast OT for Latent Domain Adaptation
\thanks{Identify applicable funding agency here. If none, delete this.}
}

\author{\IEEEauthorblockN{Siddharth Roheda$^\star$, Ashkan Panahi$^\dagger$, Hamid Krim$^\star$}
\IEEEauthorblockA{$^\star$ \textit{Electrical and Computer Engineering Department, North Carolina State University} \\
\{sroheda, ahk\}@ncsu.edu\\
$^\dagger$ \textit{Dept. of Computer Science and Engineering, Chalmers University}\\
ashkan.panahi@chalmers.se
}}


\maketitle

\begin{abstract}
In this paper, we address the problem of unsupervised Domain Adaptation. The need for such an adaptation arises when the distribution of the target data differs from that which is used to develop the model and the ground truth information of the target data is unknown. We propose an algorithm that uses optimal transport theory with a verifiably efficient and implementable solution to learn the best latent feature representation. This is achieved by minimizing the cost of transporting the samples from the target domain to the distribution of the source domain.
\end{abstract}


\section{Introduction}
Adapting a classifier trained on a source domain to recognize instances from a new target domain is an important problem of increasing research interest \cite{da_survey1, da_survey2, da_}. Difficulties often arise in practice, as is the case when the data is different from that which is used to train a model. Specifically, consider an inference problem where a model is learned using a certain source domain $X_s$ with the corresponding labels $Y_s$ and is used to classify samples from the target domain $X_t$ with the corresponding labels $Y_t$. Domain adaptation is required when $P(Y_s|X_s) \approx P(Y_t|X_t)$, but $P(X_s)$ is significantly different from $P(X_t)$. 

Such a shift in data distribution is seen and addressed in almost every field ranging from Natural Language Processing (NLP) to Object Recognition. Given labeled samples from a source domain, there are two groups that any Domain Adaptation (DA) approach can be classified into, i) semi-supervised DA: some samples in the target domain are labeled or ii) unsupervised DA: none of the samples in the target domain are labeled. 

Several works \cite{tda1, tda2, tda3} have demonstrated the effects of the divergence between the probability distributions of domains.These works have led to solutions of transforming the data from the target domain so as to make the associated distribution as close as possible to that of the source domain. This allows the application of the classifier trained on the source domain to classify data from the target domain post transformation. In \cite{m3da} an approach for multi-source domain adaptation was proposed to transfer knowledge learned from multiple labeled sources to a target domain by aligning moments of their feature distributions, while \cite{advda} uses a GAN to learn the transformation from the target domain to source domain. In \cite{coral1, coral2}, the authors simply align the second order statistics of the source and target domains.

\textbf{Contributions:} In this paper, we address the problem of unsupervised DA. We build on the existing works having led to various techniques including recent  generative adversarial networks \cite{gan},  to rather propose Optimal Transport for some of its advantages as a viable path to adapt the model toward classifying the target domain data. We first seek the latent representations of source and target domains to subsequently minimize the optimal transport cost. These representations for the source and target can be classified using a common classifier trained on the source data. Furthermore, we also demonstrate that it is also crucial to ensure optimal performance that $P(\hat{Y}_s|X_s) \approx P(\hat{Y}_t|X_t)$, where $\hat{Y}_s$ and $\hat{Y}_t$ are the predictions made by the classifier on the source and target domain respectively.
\section{Related work}

\subsection{Generative modeling}

The Generative Adversarial Network was first introduced by Goodfellow et al. \cite{gan} in 2014. In this framework, a generative model is pitted against an adversary: the discriminator. The generator aims to deceive the discriminator by synthesizing realistic samples from some underlying distribution. The discriminator on the other hand, attempts to discriminate between a real data sample and that from the generator. Both models are approximated by neural networks. When trained alternatively, the generator learns to produce random samples from the data distribution which are very close to the real data samples. Following this, Conditional Generative Adversarial Networks (CGANs) were proposed in \cite{cgans}. These networks were trained to generate realistic samples from a class conditional distribution, by replacing the random noise input to the generator by some useful information. As a result, the generator now aims to generate realistic data samples, when given the conditional information. CGANs have been used to generate random faces when given facial attributes \cite{face_gan} as well as to produce relevant images given text descriptions \cite{texttoim}.  

Many works have recently attempted to use GANs for performing domain adaptation. In \cite{advda} the authors use the generator to learn the features for classification and the  discriminator to differentiate between the source and target domain features produced by the generator. Figure \ref{CGAN} depicts the block diagram for this approach. In \cite{cycitoi} a cyclic GAN was used to perform image translation between unpaired images. In \cite{cycada} a cyclic GAN was implemented to adapt semantic segmentation of street images from GTA5 to CityScapes data. 
\begin{figure}
	\centering
	\includegraphics[width=0.3\textwidth]{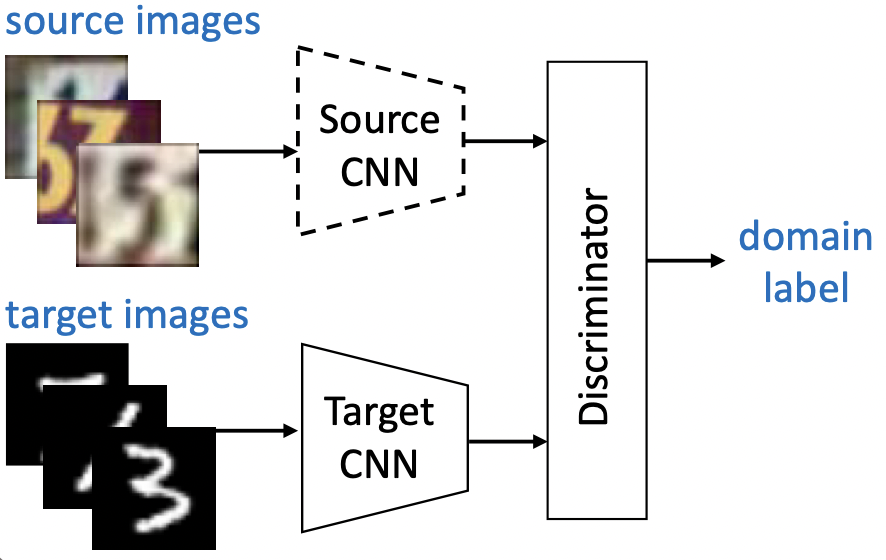}
	\caption{Adversarial Adaptation}
	\label{CGAN}
\end{figure}

\subsection{Optimal Transport}
Optimal Transport \cite{vilani} is a pointwise comparative analytical tool that provides a distance measure between two probability distributions. The distance measure is based on a cost $c(\cdot, \cdot)$ which is imputed to transporting a source distribution to a target distribution. Formally, given two densities $\mu_s$ and $\mu_t$ on two measureable spaces $\mathcal{X}_s$ and $\mathcal{X}_t$, the Kantorovich-Monge relaxation/formulation\footnote{We refer the reader to the vast literature retracing the reformulation Monge's original problem, and a very readable resource is the manuscript by Cuturi and Peyre \cite{Cut-Pey}} of the optimal transport problem entails finding a transport plan which is a probabilistic coupling $\gamma^\star$ defined over $\mathcal{X}_s \times \mathcal{X}_t$ such that,

\begin{equation}
	arg \min_{\gamma \in \Gamma} \int_{\mathcal{X}_s \times \mathcal{X}_t} c(x_s,x_t) d\gamma(x_s,x_t), 
\end{equation}

where $c: \mathcal{X}_s \times \mathcal{X}_t \to [0 , +\infty]$ and $c(x,y)$ denotes the cost of transporting a unit of mass from $x$ to $y$. $\gamma^\star (x,y)$ denotes the coupling that provides the minimum $\mathbb{E}_{x,y \texttildelow \gamma^\star [c(x,y)] }$.

In most practical applications, one has access only to the samples of the distribution where discrete measures $\mu_s = \sum_{i=1}^{N_s} p_{s_i} \delta_{x_{s_i}}$ and $\mu_t = \sum_{i=1}^{N_t} p_{t_i} \delta_{x_{t_i}}$, where $\delta_{x_{s_i}}$, $p_{s_i}$ and $\delta_{x_{t_i}}$, $p_{t_i}$ denote the Dirac function and the probability mass at $x_{s_i} \in \mathcal{X}_s$ and $x_{t_i} \in \mathcal{X}_t$, respectively. The optimal transport plan under the discrete case is the solution to a linear programming problem which is defined as follows,
\begin{equation}
	\gamma^\star = arg \min_{\gamma \in \Gamma} <\bm{C}, \gamma> = 	arg \min_{\gamma \in \Gamma} \sum_{i=1}^{N_s} \sum_{j=1}^{N_t} \gamma_{ij} C_{ij},
\end{equation}
where $\bm{C} \geq \bm{0}$ is the cost matrix with $C_{ij} = || (x_{s_i} - x_{t_i}) ||_2^2$ and
\begin{equation}
	\Gamma = \{ \gamma \in \mathbb{R}_+^{N_s \times N_t} | \gamma \bm{1}_{N_s} = \mu_s, \gamma^T\bm{1}_{N_t} = \mu_t \}.
\end{equation}
is the set of probabilistic coupling matrices and $\mathbf{1_{\cdot}}$ is a vector of ones of appropriate dimension. 

\section{Problem Formulation}
Consider data from a source domain, $X_s = \{x_{s_i}\}_{i=1,...N}$ with a corresponding set of labels $Y_s = = \{y_{s_i}\}_{i=1,...N}$, where N is the total number of samples in the dataset. Let $g_s: X_s \to L_s$ be a function that transforms the data into a latent feature space, $L_s = g_s(X_s)$. Following this, a classifier function $f(.)$ is used to assign a labels to the data samples, $\hat{Y}_s = f(L_s) = f(g_s(X_s))$. If the classifier is well trained, $\hat{Y}_s \approx Y_s$.

Now, consider target domain data $X_t$ for which the ground truth labels are unavailable. One may consider using the classifier trained on $X_s$ to classify the data $X_t$ if similar classes as in the source domain are of interest. Such a procedure would yield optimal performance if and only if the distributions of $X_s$ and $X_t$ are the same. This usually fails to be the case in practical applications, and hence resulting in sub optimal classification performance.

In order to mitigate this problem, Domain Adaptation (DA) is required. Note that our goal here is to take on the classification problem where labels for the target distribution are completely unknown, and hence to learn the function $g_t: X_t \to L_t$ such that $Y_t = f(g_t(X_t))$ leads to optimal classification performance in the absence of any information about the target domain.

\section{Proposed Approach}
\begin{figure*}
	\centering
	\includegraphics[width = 0.65\textwidth]{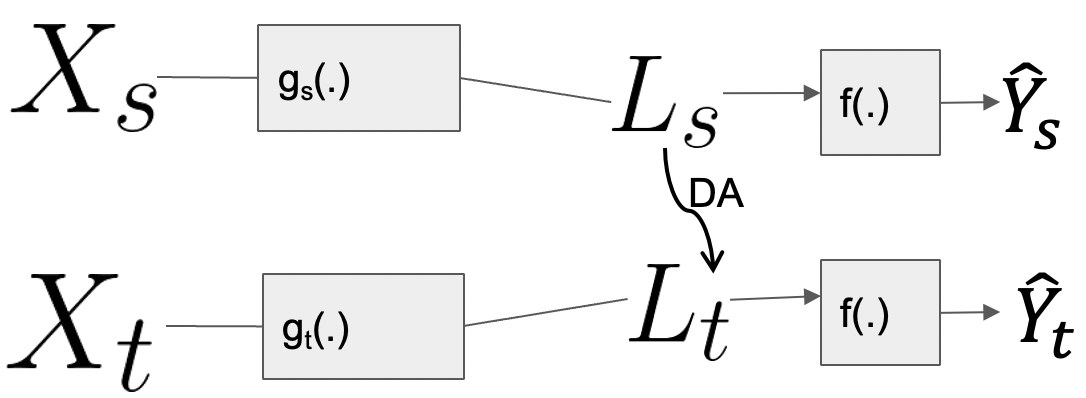}
	\centering
	\caption{Block diagram for the proposed approach}
	\label{bd_otda}
\end{figure*}
As noted in the previous section, the inference model must be optimal for the source domain. In order to ensure this, we propose to learn the functions $g_s(.)$ and $f(.)$ so that they minimize the cross entropy loss between the ground truth labels, $Y_s$ and those predicted by the model, $\hat{Y}_s = f(g_s(X_s))$,
\begin{equation}
	\min_{f, g_s}  C_{Loss} (Y_s, f(g_s(X_s))) ,
\end{equation}
where, $ C_{Loss}(Y_s, f(g_s(X_s))) = \sum_{i = 1}^{N} -y_{s_i} \log f(g_s(x_{s_i})).$

\subsection{Learning the Optimal Latent Space}
We first aim to learn the optimal latent spaces $L_s$ and $L_t$ such that the same classifier be used for both the source and target by minimizing the cost of transporting the samples in the latent space of source domain to that of the target domain. This leads to learning of latent spaces $L_s = g_s(X_s)$ and $L_t = g_t(X_t)$ with minimum discrepancy. The function that must be optimized is given as,
\begin{equation}
	\min_{f, g_s, g_t}  C_{Loss} (Y_s, \hat{Y}_s) + \lambda_1 T_{Loss}(L_s, L_t) , 
	\label{TLoss}
\end{equation}
where, $T_{Loss}(L_s,L_t) = \sum_{i,j} \gamma^\star_{ij}C_{ij}$. $\gamma^\star$ is the optimal transport mapping for going from $L_{t_i}$ to $L_{s_j}$, and $C_{ij}$ is the corresponding cost. The determination of $\gamma^\star$ is further discussed in Section \ref{eff_optimization}. $\lambda_1 $ in Equation \ref{TLoss} is a hyperparameter which controls the importance of the second term with respect to the first one. 

To ensure an optimal adaptation of the source domain classifier to that of the target domain, we proceed to minimize the cost for transporting between $L_s$ and $L_t$, while safeguarding the invariance of the predictive  power of the source domain classifier $f(\cdot)$ when applied to target domain, i.e. $P(Y_s|L_s) \approx P(Y_t|L_t)$. To best integrate this constraint, we opt to include the classification  cost in the loss to be optimized as a nonlinear  transformation of the latent representation  of the input data. A probabilistic interpretation of such an augmentation of the divergence/discrepancy loss, as the energy function of a Boltzmann distribution. In sum, the overall transportation loss may be written as, 
\begin{equation}
	\min_{f, g_s, g_t}  C_{Loss} (Y_s, \hat{Y}_s) + \lambda_1 T_{Loss}([L_s; f(L_s)], [L_t; f(L_t)]). 
	\label{TLoss_labels}
\end{equation}
Hence, $C_{ij}$ is now the cost of transporting the vector $[L_{t_i}; f(L_{t_i})]$ to $[L_{s_j}; f(L_{s_j})]$.

While the source domain data and their associated labels regularize the problem and reduce the search space, it is also practically important to regularize the target domain whose labels are unknown. To that end, we constrain the entropy of the predicted target labels, thereby not only tying the  target latent representation but also its associated labels, thus reducing the search space again. This entropy criterion in some sense encourages the model to make more confident decisions on the target space, resulting in the overall transport loss as,
\begin{equation}
\begin{split}
		\min_{f, g_s, g_t}  C_{Loss} (Y_s, \hat{Y}_s) + \lambda_1 T_{Loss}([L_s; f(L_s)], [L_t; f(L_t)]) \\
		+ \lambda_2 H(\sigma(f(g_t(X_t)))), 
\end{split}
\label{Full_opt}
\end{equation}
where, $\sigma$ is the softmax function, $H(z) = \sum_{i} - z_i \log z_i$, and $\lambda_1 \text{ and } \lambda2$ control the contributions of the last two terms.

\subsection{Discrete Optimal Transport Problem: An Efficient Resolution}
\label{eff_optimization}
The solution of the proposed model evolves along two directions: the first solves for the optimal transport map $\gamma^\star$ while keeping the functions $g_s(\cdot), g_t(\cdot),$ and $f(\cdot)$ constant, and the next one solves for the functions $g_s(\cdot), g_t(\cdot),$ and $f(\cdot)$ as in Equation \ref{Full_opt}.  

In order to compute $\gamma^\star$, we proceed to solve the following optimization problem,
\begin{gather}
	\nonumber \gamma^\star = arg \min_{\gamma \in \mathbb{R}_{+}^{N_s \times N_t}	} \sum_{i,j} \gamma_{ij} C_{ij} \\
	subject \text{ } to: \gamma^T\mathbf{1} = \mu_s; \text{ } \gamma \mathbf{1} = \mu_t	 
	\label{primalgamma_star}
\end{gather}
Assuming $N_s=N_t=N$, there are $N^2$ unknowns and $2N$ constraints. This at best leads to a computational complexity of $O(N^3\log N)$. to address this difficulty, accounting for the fact that the transport plan $\gamma$ is sparse (at most $2N-1$ non-zero elements) and its support known, would reduce the search space. The dual of Equation \ref{primalgamma_star}, which is the dual discrete form of the Kantorovic formulation can be written as,
\begin{gather}
	\nonumber \max_{\phi, \psi} \sum_{i=1}^{N_s} \phi_i \mu_{s_i} + \sum_{j=1}^{N_t} \psi_j \mu_{t_j} \\
	subject \text{ } to: \phi_i + \psi_j \leq C_{ij} \forall i,j,
	\label{dualgamma_star}
\end{gather}
where the support is $(i,j)$ for which $\phi_i + \psi_j = C_{ij}$. We now have 2N unknowns, but there are $N^2$ constraints. This leads to the same difficulty as in the primal formulation. But, what the dual formulation does allow is the possibility of a Stochatic Gradient Descent (SGD) based approach to perform the optimization.

\begin{prop}
	The dual problem is of the form, 
	\begin{gather}
		\nonumber \min_x \sum_i m_i(x);\text{ } \\
		s.t: x \in \cap_{k=1}^K S_k
	\end{gather}
\end{prop}
\begin{proof}
	Set $x = \psi_j$ for some $j \in \{1,...,N\}$ and let $m_i(x) = -[\phi_i \mu_{s_i} + \psi_j \mu_{t_j}]$. If $S_k$ is the half-space defined by $\phi_i + \psi_j \leq C_{ij}$, we get the dual of the Kantorovic formulation deescribed in Equation \ref{dualgamma_star}. 
\end{proof}
SGD relies on approximation of the gradient $\sum_{i=1}^{N} \nabla m_i(x)$. This is carried out by randomly selecting $i \in [1,...,N]$ and applying $x^{t+1} = proj_S(x^t - \lambda \nabla m_i(x))$. This estimated gradient has a high variance accross samples. In order to stabilize the updates it is important to store the observed gradients in a cumulative manner to improve the estimate of overall gradient. This is done by using a Stochastic Variance Reduction (SVR) methods \cite{ashkan_paper}, $x^{t+1} = proj_S(x^t - \lambda r (\nabla m_i(x)))$, where $r$ refines the gradient estimate. Applying this to Equation \ref{dualgamma_star} the following updates must be performed,
\begin{gather}
	\hat{\phi}_i^{t+1} = \phi_i^t - \lambda r_{\phi}(\mu_{s_i}), \\
	\hat{\psi}_j^{t+1} = \psi_j^t - \lambda r_{\psi}(\mu_{t_j}).
\end{gather} 
The solution is then found by projection onto the half space $\phi_i + \psi_j \leq C_{ij}$. The support for $\gamma$ is now $(i,j)$ for which $\phi_i + \psi_j = C_{ij}$. The optimization for Equation \ref{primalgamma_star} is now over $(i,j) \in A$, where $|A| \leq N_1+N_2-1$, hence significantly reducing the computational complexity. 
\section{Experiments and Results}
\subsection{Domain Adaptation for Computer Vision}
In order to substantiate the described approach we evaluate it on various public datasets that have commonly been used in the literature to demonstrate Domain Adaptation. In each case a CNN is used to realize the functions $g_s(.)$, $g_t(.)$, and $f(.)$. 

The first dataset utilized includes handwritten digits from MNIST, USPS, and SVHN that are to be recognized, with all of the 10 classes. MNIST is used as the source domain and USPS and SVHN are considered the target domain. Figure \ref{MNIST_USPS_SVHN} provides an example of the samples in the dataset. The performance on this dataset is summarized in Tables \ref{Performance_MNIST_USPS} and \ref{Performance_MNIST_SVHN}. As can be observed, the proposed approach boosts the performance in comparison to the state of art in Domain Adaptation. It is also demonstrated that ensuring the minimization of transport cost between the predictions on source and target labels is critical towards achieving a succesful adaptation to the target domain. Figure \ref{Closs} depicts the classification loss for the source (MNIST) and target (USPS) during training. Note that the ground truth labels for the target domain are not used in the training process, but are only used to calculate the classification loss in order to visualize the training progress.

\begin{figure}
	\centering
	\includegraphics[width=0.45\textwidth]{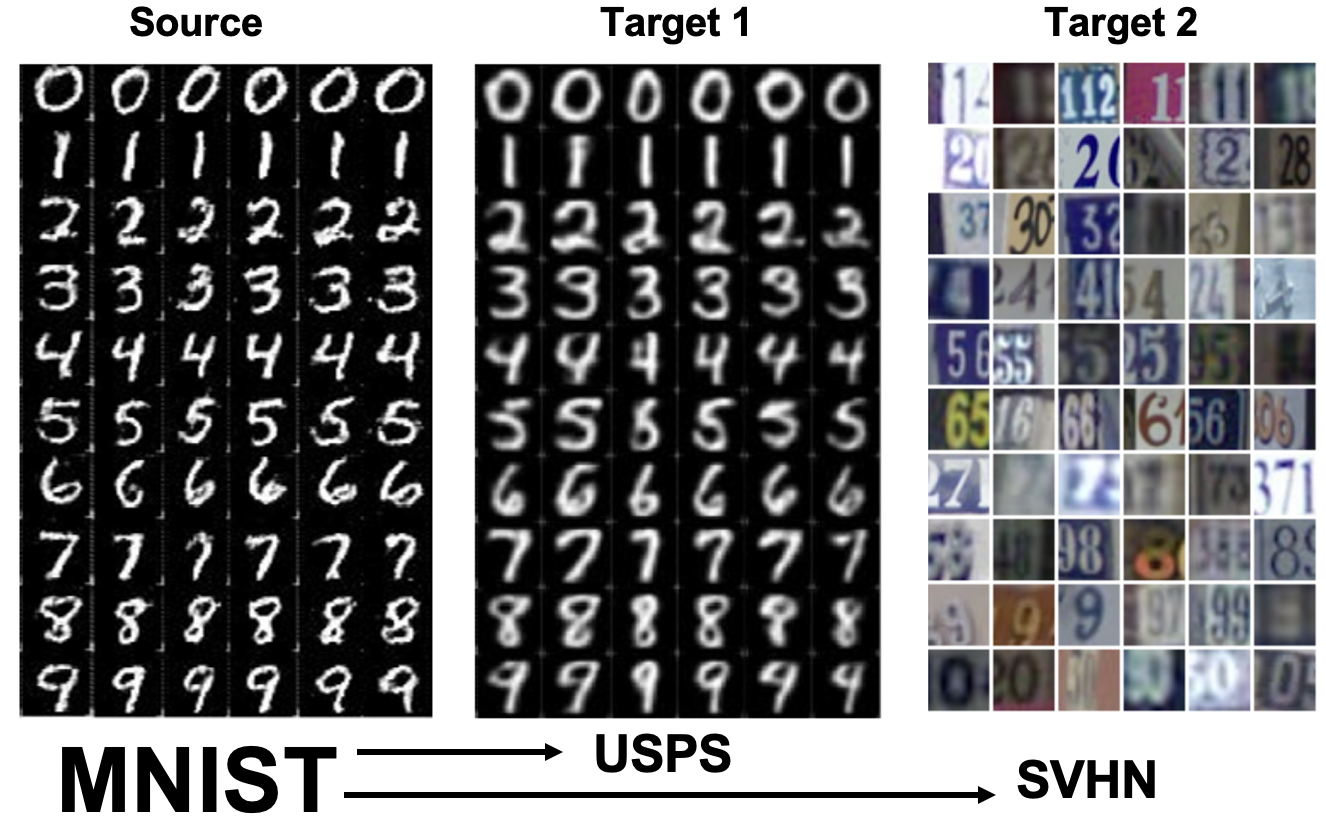}
	\caption{Samples from MNIST,USPS, and SVHN}
	\label{MNIST_USPS_SVHN}
\end{figure}

\begin{figure}
	\centering
	\includegraphics[width=0.45\textwidth]{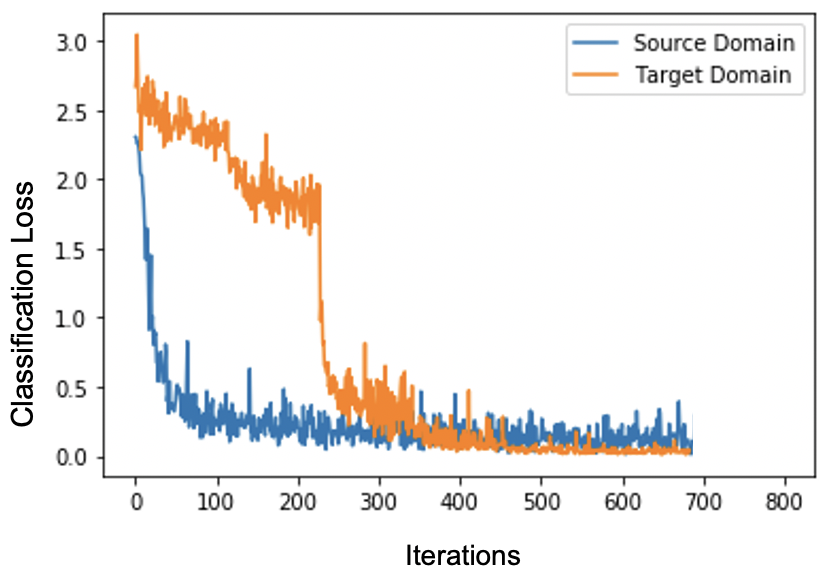}
	\caption{Samples from MNIST,USPS, and SVHN}
	\label{Closs}
\end{figure}

The second dataset that has been evaluated is a slightly more challanging one. This involves an object classification task and consist of images from Amazon (31 classes), DSLR (31 classes), Webcam (31 classes), and Caltech10 (10 classes). The Amazon and DSLR data have higher resolution images while the Webcam and Caltech10 has a lower resolution. Figure \ref{ADWC} shows the examples from this dataset. Table \ref{Indi_Performance_Amazon_DSLR_Webcam_Caltech} shows the performance of each of the domains when a seperate classifier is trained on each of them, with all the ground truth labels assumed to be available. Table \ref{Performance_Amazon_DSLR_Webcam_Caltech} demonstrates the performance when domain adaptation was used assuming the ground truth labels are only available for the source. In each case, the OT-inspired approach demonstrates a superior adaptation performance.
\begin{figure*}[tbp]
	\centering
	\includegraphics[width = 0.75\textwidth]{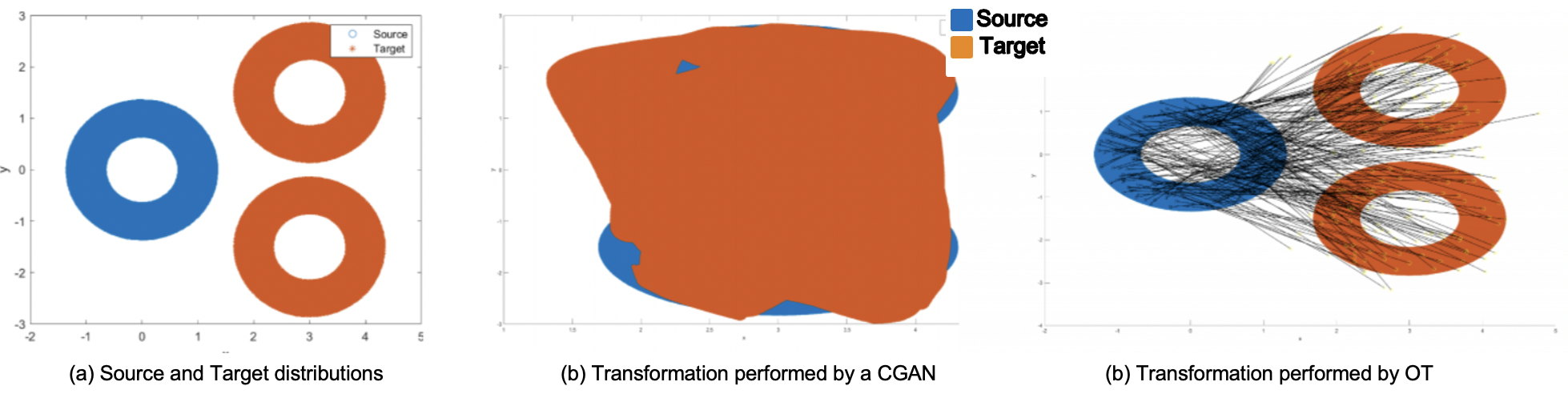}
	\centering
	\caption{An example comparing OT with a CGAN for Shape Morphing}
	\label{circ_2circ}
\end{figure*}
\begin{figure}[h!]
	\centering
	\includegraphics[width=0.5\textwidth]{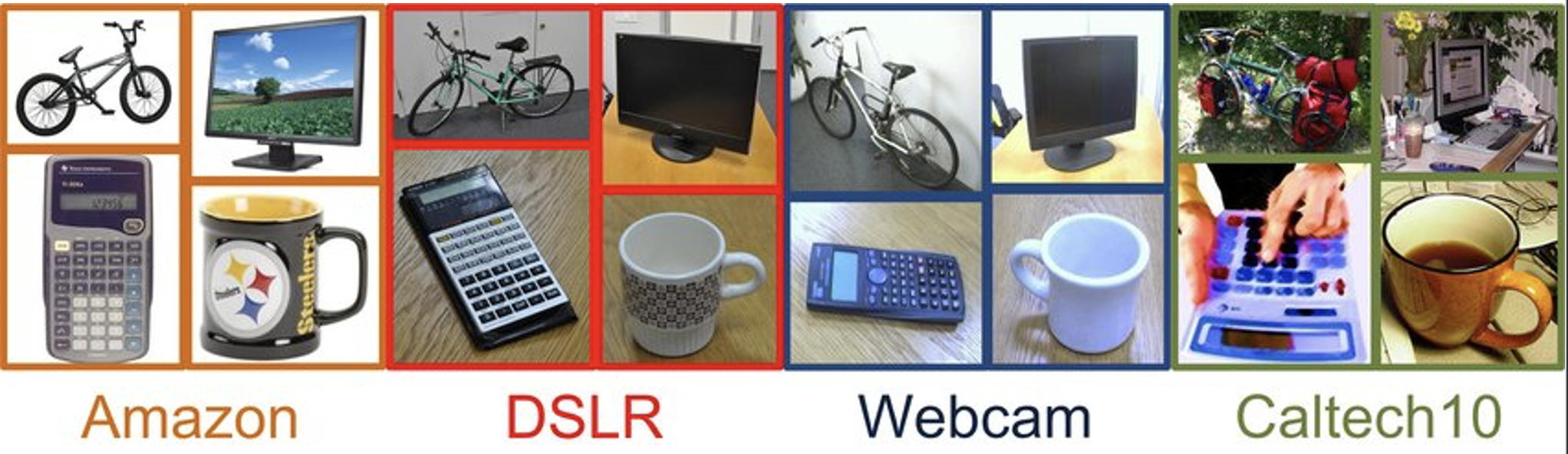}
	\caption{Samples from Amazon, DSLR, Webcam, and Caltech10}
	\label{ADWC}
\end{figure}

\begin{table}
	\begin{center}
		\begin{tabular}{|M{4.1cm}|M{1.5cm}|}
			\hline
			Dataset & Accuracy \\
			\hline
			MNIST & 99.3 \% \\
			\hline
			USPS & 98.4 \% \\
			\hline
			MNIST $\to$ USPS (OT without Labels) & 62.3 \% \\
			\hline
			MNIST $\to$ USPS (OT with Labels) & 87.2 \% \\
			\hline
			\textbf{MNIST $\to$ USPS (OT with labels + Entropy)} & \textbf{96.6 \%} \\
			\hline
			ADDA \cite{advda} & 96.0 \% \\
			\hline
		\end{tabular}
	\end{center}
	\caption{Performance on MNIST $\to$ USPS} 
	\label{Performance_MNIST_USPS}
\end{table}

\begin{table}
	\begin{center}
		\begin{tabular}{|M{3.8cm}|M{2cm}|M{2cm}|}
			\hline
			Dataset & Accuracy (Dot Product) & Accuracy (Wasserstein Distance) \\
			\hline
			MNIST & 99.3 \% & 99.3 \% \\
			\hline
			SVHN & 91.8 \% & 91.8 \% \\
			\hline
			MNIST $\to$ SVHN (OT without Labels) & 48.6 \% & 51.3 \% \\
			\hline
			SVHN $\to$ MNIST (OT without Labels) & 59.4 \% & 66.6 \% \\
			\hline
			MNIST $\to$ SVHN (OT with Labels) & 66.2 \% & 71.6 \% \\ 
			\hline
			SVHN $\to$ MNIST (OT with Labels) & 71.9 \% & 77.9 \% \\
			\hline
			\textbf{MNIST $\to$ SVHN (OT with labels + Entropy)} & \textbf{71.7 \%}  & \textbf{77.2 \%} \\
			\hline
			\textbf{SVHN $\to$ MNIST (OT with labels + Entropy)} & \textbf{78.1 \%}  & \textbf{88.4 \%} \\
			\hline
			MNIST $\to$ SVHN (ADDA \cite{advda}) & - & 76.0 \% \\
			\hline
			MNIST $\to$ SVHN (ADDA \cite{advda}) & - & 86.2 \% \\
			\hline
		\end{tabular}
	\end{center}
	\caption{Performance on MNIST $\to$ SVHN} 
	\label{Performance_MNIST_SVHN}
\end{table}

\begin{table}
	\begin{center}
		\begin{tabular}{|M{3.8cm}|M{2cm}|}
			\hline
			Dataset & Accuracy \\
			\hline
			Amazon & 64.2 \% \\
			\hline
			DSLR & 96.1 \% \\
			\hline
			Webcam & 98.6 \% \\
			\hline
			Caltech10 & 82.7 \% \\
			\hline
			
		\end{tabular}
	\end{center}
	\caption{Performance on Amazon, DSLR, Webcam, Caltech10 when all labels are available}  
	\label{Indi_Performance_Amazon_DSLR_Webcam_Caltech}
\end{table}

\begin{table}
	\begin{center}
		\begin{tabular}{|M{3.8cm}|M{2cm}|M{2cm}|}
			\hline
			Adaptation & Accuracy (OT) & Accuracy (ADDA \cite{advda}) \\
			\hline
			Amazon $\to$ Webcam
 & 86.0 \% & 75.1 \% \\
			\hline
			DSLR $\to$ Webcam & 97.6 \% & 97.0 \% \\
			\hline
			Webcam $\to$ Amazon & 91.2 \% & 88.3 \% \\
			\hline
			DSLR $\to$ Amazon & 90.6 \% & 87.4 \% \\
			\hline
			Amazon $\to$ DSLR & 99.4 \% & 99.0 \% \\
			\hline
			Webcam $\to$ DSLR & 100.0 \% & 
99.6 \% \\
			\hline
			Amazon $\to$ Caltech10
& 87.6 \% & 84.8 \% \\ 
			\hline
			DSLR $\to$ Caltech10
& 86.5 \% & 81.2 \% \\ 
			\hline
			Webcam $\to$ Caltech10
& 82.8 \% & 75.0 \% \\ 
			\hline
		\end{tabular}
	\end{center}
	\caption{Domain Adaptation on Amazon, DSLR, Webcam, Caltech10}  
	\label{Performance_Amazon_DSLR_Webcam_Caltech}
\end{table}

\subsection{Shape Morphing}
In addition to the computer vision applications, we also evaluate the optimal transport approach detailed in Section \ref{eff_optimization} toward shape morphing, and compare it with a Conditional Generative Adversarial Network (CGAN). Shape morphing is the task of converting a source shape defined by points $x_s \in \mathcal{X}_s$ into a target shape $x_t \in \mathcal{X}_t$. If we consider $N_1=N_2=1000$, a standard method minimizing Equation \ref{primalgamma_star} requires memory of 8 MB while the approach detailed in Section \ref{eff_optimization} requires 36 KB. We consider Optimal Transport between curves on $\mathbb{R}^2$ by treating them as distributions. Figure \ref{circ_sq} shows the case with a circular curve as the source and a square as the target along with the computed transport maps.
\begin{figure}
	\centering
	\includegraphics[width=0.5\textwidth]{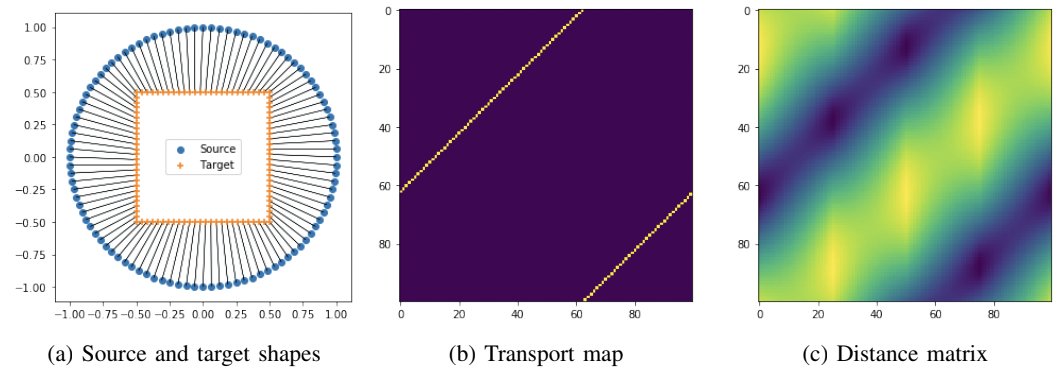}
	\caption{An example for transforming a circle to a square}
	\label{circ_sq}
\end{figure}
In Figure \ref{circ_2circ} we show the case when there are two circles in the target and the source is a single circle. The result of such a transform is consequently compared with a CGAN and it can be observed that the CGAN fails in such a case most likely due to the discontinuity in the input. The failure of the GAN to perform as expected is due to its inherent stability issues which are discussed and addressed in detail in \cite{Krim}.

\section{Conclusion}
In this paper we proposed a Domain Adaptation approach based on Optimal Transport theory with a new efficient OT algorithm with demonstrably greater computational and effective performance, ensuring that a classifier trained on some source domain can still perform at or better than state of the art classification on a target domain that has a different data distribution. We also show that it is important to consider the distribution of the model predictions when learning the transport map. Finally we compare the performance of this OT Domain Adaptation, with the Adversarial Domain Adaptation and show that we can outperform them using this approach. Furthermore we also demonstrate the strength of OT when it comes to shape morphing in comparison to a CGAN. 

\vspace{12pt}

\end{document}